\pdfoutput = 1
\documentclass[11pt]{article}
\usepackage[utf8]{inputenc}
\usepackage{amssymb,amsmath,amsthm}
\usepackage{algpseudocode}
\usepackage{algorithm}
\usepackage{bbm}
\usepackage{xifthen}
\usepackage{libertine} %
\usepackage[letterpaper,margin=1in]{geometry}
\usepackage{natbib}
\usepackage{booktabs}
\usepackage{float}
\usepackage{xspace}
\usepackage{xcolor}
\usepackage{thm-restate}
\usepackage[unicode=true,pagebackref=true]{hyperref}       %
\definecolor{mydarkblue}{rgb}{0,0.08,0.45}
\hypersetup{ %
    pdftitle={},
    pdfauthor={},
    pdfsubject={},
    pdfkeywords={},
    pdfborder=0 0 0,
    pdfpagemode=UseNone,
    colorlinks=true,
    linkcolor=mydarkblue,
    citecolor=mydarkblue,
    filecolor=mydarkblue,
    urlcolor=mydarkblue,
}
\usepackage{cleveref}
\usepackage{listings}

\newtheorem*{definition*}{Definition}
\newtheorem*{proposition*}{Proposition}
\newtheorem*{corollary*}{Corollary}
\newtheorem{theorem}{Theorem}[section]
\newtheorem{lemma}[theorem]{Lemma}
\newtheorem{definition}[theorem]{Definition}

\newtheorem{remark}[theorem]{Remark}

\newcommand{\eat}[1]{}

\newcommand{\R}{\mathbb{R}}
\newcommand{\E}{\mathop{\mathbb{E}}}
\newcommand{\N}{\mathbb{N}}
\newcommand{\Z}{\mathbb{Z}}

\newcommand{\clip}{\mathrm{clip}}
\newcommand{\argmin}{\mathop{\arg\min}}

\newcommand{\zo}{[0,1]}
\newcommand{\bits}{\{0,1\}}

\newcommand{\ber}{{\mathsf{Ber}}}
\newcommand{\dual}{{\mathsf{dual}}}

\newcommand{\X}{\mathcal{X}}
\newcommand{\Y}{\mathcal{Y}}
\newcommand{\mC}{\mathcal{C}}

\newcommand{\mF}{\mathcal{F}}
\newcommand{\mD}{\mathcal{D}}
\newcommand{\mG}{\mathcal{G}}

\newcommand{\mJ}{\mathcal{J}}

 \newcommand{\NN}{\mathsf{NN}}
 \newcommand{\relu}{\mathsf{ReLU}}

\newcommand{\pmo}{\ensuremath{ \{\pm 1\} }}

\newcommand{\opt}{\mathrm{OPT}}

\newcommand{\junta}{{\mathcal{J}}}
\newcommand{\maj}{{\mathsf{MAJ}}}

\newcommand{\sps}[1]{^{(#1)}} %

\newcommand{\sq}{{\ell_{\mathsf{sq}}}}
\newcommand{\xent}{{\ell_{\mathsf{xent}}}}
\newcommand{\logistic}{{\ell_{\mathsf{logistic}}}}
\newcommand{\eloss}{L}

\def\shownotes{1}  
 \ifnum\shownotes=1
\usepackage{todonotes}
\else
\usepackage[disable]{todonotes}
\fi

\newcommand{\pnote}[1]{\todo[linecolor=blue,backgroundcolor=cyan!25,bordercolor=blue,inline]{#1 --PG}}

\usepackage{pifont}%

\usepackage{cleveref}
\usepackage{graphicx}
\usepackage{subcaption}
\usepackage{listings}
\usepackage{multicol}
\usepackage{wrapfig}

\newcommand{\lunjia}[1]{\textcolor{orange}{[LH: #1]}}

\title{Loss Minimization Yields Multicalibration\\ for Large Neural Networks}

\author{
{}\and
Jaros\l aw B\l asiok\\
ETH Zürich \and
Parikshit Gopalan\\
Apple \and
Lunjia Hu\\
Stanford University \and
{}\and
Adam Tauman Kalai\\
Microsoft Research \and
Preetum Nakkiran\\
Apple 
}
\date{}

\allowdisplaybreaks
\begin{document}

\maketitle

\newcommand{\intzo}{\ensuremath{ [0,1]}}

\begin{abstract}
Multicalibration is a notion of fairness for predictors that requires them to provide calibrated predictions across a large set of protected groups. Multicalibration is known to be a distinct goal than loss minimization, even for simple predictors such as linear functions. 

In this work, we consider the setting where the protected groups can be represented by neural networks of size $k$, and the predictors are neural networks of size $n > k$. We show that minimizing the squared loss over all neural nets of size $n$ implies multicalibration for all but a bounded number of \emph{unlucky} values of $n$. We also give evidence that our bound on the number of unlucky values is tight, given our proof technique.  Previously, results of the flavor that loss minimization yields multicalibration were known only for predictors that were near the ground truth, hence were rather limited in applicability. Unlike these, our results rely on the expressivity of neural nets and utilize the representation of the predictor.
\end{abstract}

\section{Introduction}

In supervised binary prediction, we are given examples $(x,y)\in \X\times \{0,1\}$ drawn independently from an unknown distribution $\mD$, where the labels $y\in \{0,1\}$ are binary. We wish to learn a predictor $f:\X\to [0,1]$ which assigns to each point $x\in \X$ a prediction $f(x)\in [0,1]$ as the estimated probability that the label is $1$. The performance of a prediction model is commonly measured using a loss function; popular losses include the  squared loss 
and the cross-entropy loss. 
\eat{Beyond performance measures, loss functions also serve as natural objectives for algorithm design. }
Machine learning algorithms learn a predictor $f$ by iteratively optimizing the expected loss (e.g.\ via stochastic gradient descent). This simple paradigm has proved remarkably powerful, and modern machine learning has a powerful arsenal of theoretical and empirical tools for loss minimization. 

Yet despite its considerable success, loss minimization alone typically does not guarantee everything we might want from our prediction models. It is not obvious if important desiderata such as fairness, privacy and interpretability can be guaranteed just from loss minimization. There has been considerable research effort dedicated to understanding how to modify the loss minimization template to ensure these desiderata, and into possible tradeoffs between these goals and loss minimization.

A desirable fairness guarantee that has been studied intensively in recent years is \emph{multicalibration}, introduced in the work of \citet{hkrr2018} \citep[see also][]{kearns2017preventing, KleinbergMR17}. Informally, it asks that the predictions be calibrated conditioned on each subgroup of the population $\X$ for a family of subgroups. 

\begin{definition}[Multicalibration \citep{kim2022universal, hkrr2018}]
\label{def:mc}
Let $\gamma > 0$ and $\mC$ be a class of auditor functions $c:\X\times [0,1] \to [-1,1]$.
The predictor $f:\X\to [0,1]$ is $(\mC, \gamma)$-multicalibrated or $(\mC, \gamma)$-MC if for all $c \in \mC$,
\begin{equation}
\label{eq:mc}
\Big|\E\nolimits_{(x,y)\sim \mD}\left[c\bigl(x, f(x)\bigr) \cdot \bigl(y-f(x)\bigr)\right]\Big| \le \gamma.
\end{equation}
When $\mC$ is clear from context, we say that $f$ is $\gamma$-MC for brevity.
\footnote{The MC definition above is version due to \citet[][Appendix 1]{kim2022universal}, which was also used by \citet{happymap, dwork2022beyond}. The original definition \citep{hkrr2018, GopalanKSZ22} considers only those $c$ that can be factored as $c(x, f(x)) = g(x)w(f(x))$ where $g:\X \to [-1, 1]$ and $w:\intzo \to [-1,1]$. For the setting where $\mC$ is a family of neural nets, the general definition is more natural, since it amounts to both $x$ and $f(x)$ being inputs to an auditor neural net.}
\end{definition}

When $c$ maps to $\{0,1\}$ values, it can be viewed as defining a {\em group} which is a subset of the domain $\X$ and MC means that $\E[y] \approx \E[f(x)]$ conditioned on membership in this group, where the degree of closeness depends on the size of the group. Good calibration is guaranteed for any group of sufficient likelihood that can be identified by the auditor family $\mC$. The connection to calibration comes from observing that taking $c(x, f(x)) =\mathbf{1}[f(x) = v]$ yields a contribution of at most $\gamma$ to the standard notion of expected calibration error (ECE) from the individuals $x$ with $f(x) = v$, and taking $c(x,f(x)) = \eta(f(x))$ for all choices of $\eta:[0,1]\to [-1,1]$ yields having ECE $\le\gamma$.

\subsection{Loss Minimization and Multicalibration}

Exploring the relation between multicalibration and loss minimization has led to a rich body of research work \citep{hkrr2018, omni, op2, op3, hu2022omnipredictors, GJRR23}. Let us describe the main findings from this body of work and how it relates to the question that we address.

The original paper of \cite{hkrr2018} showed that multicalibration can be achieved via a post-processing procedure: we can take an arbitrary predictor $f$ and {\em boost} it to a $\mC$-multicalibrated predictor $f'$ such that the squared loss (or any proper loss) only decreases. This suggests that multicalibration need not come at the cost of increased expected loss, provided we are willing to consider models of greater complexity (than $f$ and $\mC$). Indeed, during the boosting, the \citep{hkrr2018} algorithm repeatedly refines the predictor by augmenting it with functions from the base class $\mC$, thus increasing its accuracy and its complexity. The connection to boosting is made even more explicit in  \citep{omni,mc-boosting}. 
Given this, a natural question which motivates our work is: 

{\bf Question 1: Can we learn a $\mC$-multicalibrated predictor by performing standard loss minimization over a family $\mC'$ with greater complexity than $\mC$?}

By standard loss minimization, we mean a procedure that aims to minimize a loss $\ell$ in expectation over a fixed class $\mC'$ of predictors. 
The works of \citet{hkrr2018, omni, mc-boosting} point us to classes $\mC'$ that do indeed contain $\mC$-multicalibrated predictors, and they tell us how to find one such predictor using boosting iterations that also decrease the expected loss. However, they do not tell us that \emph{directly} minimizing any particular loss will lead us to one such predictor: the loss minimizer might prefer other predictors within $\mC'$ with lower expected loss, but which are not multicalibrated. 
Indeed we know examples where this tradeoff between loss minimization and multicalibration happens for some classes $\mC'$ \citep{gopalan2021multicalibrated, calibgap}. 
Thus Question 1 asks if there is a more direct connection between loss minimization and multicalibration than what was previously known.

If we increase the power of $\mC'$ to the extent that the family is expressive enough to contain the ground-truth predictor $f^*(x):= \E_\mD[y|x]$ (the realizable case), for the squared loss and the cross-entropy loss, it is not hard to see that loss minimization over the family will bring our predictor $f$ close to the ground truth $f^*$, and we will get multicalibration as a consequence \citep[see][Chapter 3]{barocas-hardt-narayanan} \citep{liu2019implicit}. However, this is not particularly insightful, since we do not expect this strong realizability assumption to hold in practice for predictor families used in common loss minimization algorithms. 

For more reasonable choices of $\mC'$ that do not guarantee realizability of the ground truth predictor, previous results reveal potential challenges to giving a positive answer to Question 1.
For some simple choices of $\mC'$, it is known that Question 1 has a negative answer since
there is a tradeoff between loss minimization and multicalibration. Indeed, for the predictor family considered in logistic regression comprising sigmoids of linear functions, 
the predictor with minimum expected loss need not be 
even calibrated \citep[see][]{calibgap}. Even if $\mC'$ contains predictor $f_1$ that is $\mC$-multicalibrated, it might contain another predictor $f_2$ that is better at loss minimization.\footnote{ Running the \citep{hkrr2018} boosting on $f_2$ will result in $f_3$ which is better at loss minimization than $f_2$ and is multicalibrated, but it may no longer belong to $\mC'$.}

Another potential challenge comes from the work of \citep{omni}, which showed that a multicalibrated predictor is an \emph{omnipredictor}, i.e., it can be used to minimize \emph{any} convex and Lipschitz loss function compared to a benchmark class of models defined based on $\mC$. \eat{Intuitively, if a predictor makes predictions that are indistinguishable from ground truth, then one can imagine that it ought to do well at minimizing any loss function. This argument is made precise in the work of \citep{op2}, via the notion of loss outcome indistinguishability which bridges the two notions above.} The omniprediction results seem to suggest that a positive answer to Question 1 without realizability is perhaps unlikely: a standard loss minimization procedure will find the best predictor $f \in \mC'$ tailored to a particular loss $\ell$, whereas in order to be $\mC$-multicalibrated, $f$ needs to be competitive with $\mC$ for every Lipschitz, convex loss. It is tempting to believe that algorithms for achieving multicalibration might have to go beyond the usual framework of minimizing a single loss function over a family of predictors. Indeed, all previously known algorithms for multicalibration require boosting updates similar to the algorithm of \citet{hkrr2018}.

To summarize, prior work tells us the following:
\begin{itemize}
     \item There are examples of $\mC'$ where loss minimization does not yield (multi)calibration. In these examples, $\mC'$ is not too much more powerful than $\mC$.
    \item Any positive answer to Question 1 requires $\mC'$ to be sufficiently powerful 
    relative to $\mC$, so that optimizing over $\mC'$ gives a $\mC$-omnipredictor for all convex, Lipschitz losses. 
    \item The answer to Question 1 is yes when $\mC'$ is extremely (perhaps unreasonably) powerful, so that it contains the ground truth predictor. 
\end{itemize}
This leaves open the possibility that a positive answer holds for $\mC', \mC$ pairs where $\mC'$ is more expressive than $\mC$, but loss minimization over $\mC'$ is still tractable.

\subsection{Our Contribution: Multicalibration from Standard Loss Minimization}

We show that multicalibration with respect to neural networks of size $k$ can be achieved solely by minimizing the squared loss over the family of neural networks of size $n$, {\bf for all but a few choices of $n$}. This result provides a positive answer to Question 1 without making any realizability assumptions. While we focus on the squared loss, our results extend to any proper loss such as the cross-entropy loss (see \Cref{sec:proper}).

Specifically, we take the auditor class $\mC$ in \Cref{def:mc} to be $\NN_k^*$ which consists of all functions $c:\X\times [0,1]\to [-1,1]$ computable by some $k$-node neural network, where we assume the domain $\X$ is a subset of a Euclidean space $\R^d$. For concreteness, we use the ReLU activation function, which is a popular choice in practice (see \Cref{sec:prelims} for formal definitions).  
Multicalibration w.r.t.\ $\mC = \NN_k^*$ guarantees good calibration on \emph{all} large enough groups identifiable by a size-$k$ neural network. This is a strong guarantee when $k$ is large enough to express interesting groups. We consider minimizing the squared loss over the family $\NN_n$ of all predictors $f:\X\to [0,1]$ computable by some $n$-node neural network. For a distribution $\mD$ over $\X\times \{0,1\}$, we say a predictor $f\in \NN_n$ is $\varepsilon$-loss-optimal if
\[
\E\nolimits_{(x,y)\sim \mD}[(f(x) - y)^2] \le \inf_{f'\in \NN_n}\E\nolimits_{(x,y)\sim \mD}[(f'(x) - y)^2] + \varepsilon.
\]
We prove the following theorem showing that loss optimality implies multicalibration for all but a few choices of $n$:
\begin{theorem}%
\label{thm:informal-main}
    Let $\mD$ be a distribution over $\X\times \{0,1\}$. For every $k\in \Z_{> 0}$ and every $\alpha > 0$, for all but at most $(k + 2)/\alpha$ choices of $n\in \Z_{\ge 0}$, for any $\varepsilon\in (0,1)$, every $\varepsilon$-loss-optimal predictor $f\in \NN_n$ is $(\NN_k^*, \sqrt{\alpha + \varepsilon})$-MC.
\end{theorem}

Note that the above theorem does not make any realizability assumptions. In particular, an $\varepsilon$-loss-optimal predictor $f\in \NN_n$ may have a significantly larger expected squared loss compared to $f^*$, and the theorem still guarantees that $f$ is multicalibrated. \Cref{thm:informal-main} implies that if a neural network $f\in \NN_n$ violates multicalibration beyond a small threshold, its expected squared loss must be sub-optimal within the family $\NN_n$, except for a few unlucky choices of $n$. Outside of these choices, a multicalibration violation indicates potential for further improvement of the expected loss {\em within the family $\NN_n$}.

Our result can be viewed as a demonstration of the representation ability of neural networks, complementary to the universal representation theorems \citep[e.g.][]{universal-approx}. While neural networks of a certain size $n$ cannot express all functions (in particular, the ground truth $f^*$ may be far from such neural networks), except for a few choices of $n$ they can always express a multicalibrated predictor, and such a predictor can be found by minimizing the squared loss. 

Our results are more about representational aspects of neural networks, and do not address algorithmic or sample complexity considerations. They currently do not apply to specific algorithms
for optimizing neural networks, such as SGD, and they should not be interpreted as
``fairness comes for free from optimizing neural networks''. The question of whether loss minimization over neural networks can be performed efficiently does not have a simple answer, SGD is found to do well in practice. The question of whether it results (for most $n$) in networks that are multicalibrated for smaller size neural networks is an interesting question from the theoretical and experimental viewpoint. See \Cref{sec:discussion} for an extended discussion.

\paragraph{A Generalization.}
In our proof of \Cref{thm:informal-main}, a key property we use about neural networks  provides an explanation for their representation ability demonstrated by the theorem. The property is the simple fact that the composition of two neural networks can be implemented by another neural network with size roughly equal to the sum of the sizes of the two initial networks. Indeed, we generalize \Cref{thm:informal-main} to any sequence of families of predictors closed under composition:

\begin{theorem}
\label{thm:sequence}
Let $\mC$ be a class of auditing functions $c:\X\times [0,1]\to [-1,1]$.
Let $\mF_0,\mF_1,\ldots$ be families of predictors $f:\X\to [0,1]$ satisfying $\emptyset \ne \mF_0\subseteq \mF_1 \subseteq \cdots$. For some positive integer $k$, assume that for every nonnegative integer $n$, every $f\in \mF_n$, every $c\in \mC$, and every $\beta\in [-1,1]$, the composed predictor $f'$ defined below satisfies $f'\in \mF_{n+k}$:
\begin{equation}
\label{eq:f'}
f'(x) = \clip(f(x) + \beta c(x,f(x))) \quad \text{for every }x\in \X,
\end{equation}
where $\clip (z) = \max(0,\min(1,z))\in [0,1]$ for every $z\in \R$. Then for every $\alpha > 0$, for all but at most $k/\alpha$ choices of $n\in \Z_{\ge 0}$, for any $\varepsilon > 0$, every $\varepsilon$-loss-optimal $f\in \mF_n$ is $(\mC, \sqrt{\alpha + \varepsilon})$-MC.
\end{theorem}

By definition, neural networks satisfy closeness under composition as required in \Cref{thm:sequence}, allowing us to prove \Cref{thm:informal-main} as a consequence of \Cref{thm:sequence} (see \Cref{sec:analysis} for details). 
In addition, \Cref{thm:sequence} implies variants of \Cref{thm:informal-main} where we enforce various network architectures, though for simplicity we focus on a general feed-forward architecture in this paper.

Besides neural networks, we can apply \Cref{thm:sequence} to other predictor families. For example, a well-studied family in computational learning theory is the family of \emph{juntas}, which is also the family we consider in our lower bound results (see \Cref{sec:intro-lower} below). Here, the domain $\X$ is the Boolean cube $\{-1,1\}^m$, and a function $f$ over $\X$ is called a $k$-junta if $f(x)$ only depends on a fixed set of $k$ coordinates of $x$ for every $x\in \X$ (see \Cref{sec:lower} for formal definition). We use $\mJ_k$ to denote the class of all $k$-juntas $f:\X\to [0,1]$, and use $\mJ_k^*$ to denote the class of all $k$-juntas $f:\X\to [-1,1]$. 
We consider multiaccuracy (MA) (defined formally in \Cref{def:ma}) which is a weaker notion than multicalibration where functions $c\in \mC$ only takes $x$ as input, instead of taking $x$ and $f(x)$ as input as in \Cref{def:mc}.
Like neural networks, juntas also satisfy closeness under composition, so we get the following theorem as a corollary of \Cref{thm:sequence} (see \Cref{sec:lower} for proof):

\begin{restatable}{theorem}{juntaupper}
\label{thm:upper-bound}
Let $\mD$ be a distribution over $\X\times \{0,1\}$, where $\X = \{-1,1\}^m$ for a positive integer $m$. Let $k\in \mathbb Z_{> 0}$ and $\alpha\in (0,1]$ be parameters. Then for all but at most $k/\alpha$ nonnegative integers $n$, for any $\varepsilon > 0$, every $\varepsilon$-loss-optimal $f\in \mJ_n$ is $(\junta_k^*, \sqrt{\alpha + \varepsilon})$-MA. 
\eat{(More precisely, every $c\in \mC$ takes $(x,v)\in \X\times [0,1]$ as input, ignores $v$, and outputs $f(x)$ for some $f\in \junta_k(X)$. Since $c$ ignores $v$, the notion of multicalibration here reduces to multiaccuracy.)}
\end{restatable}

\paragraph{A Smoothed Analysis Perspective.} 
Since there are a few unlucky neural network sizes $n$ for which \Cref{thm:informal-main} does not provide multicalibration guarantees, one may be worried that the size used in a  task in practice might be among the unlucky ones. However, it is often reasonable to assume that there is sufficient randomness involved in a specific choice of $n$ in practice, especially when $n$ is large as in modern neural networks. If $n$ is chosen, say, uniformly at random from a large range, the probability of $n$ being among the few unlucky choices is small. This is the perspective taken in \emph{smoothed analysis} \citep{smoothed-analysis}: by assuming that the instances arising in practice contain some intrinsic randomness, we can often show stronger guarantees for them than for the worst-case instance. We also show how to completely avoid unlucky sizes by adding a regularization term to the loss minimization problem (see \Cref{sec:srm}).

\paragraph{Proof Sketch.} Our proof of \Cref{thm:sequence} is simple in hindsight, which we view as a plus. It combines the existing analysis of the boosting style updates of \citet{hkrr2018} with some new ideas, and exploits the structure of the updates. 

For $\gamma = \sqrt{\alpha + \varepsilon}$, consider a predictor $f:\X\to [0,1]$ that is not $\gamma$-multicalibrated w.r.t.\ $\mC$, i.e., there exists $c\in \mC$ such that \eqref{eq:mc} is violated. A key step in previous boosting style algorithms is to use $c$ to decrease the expected squared loss of $f$, by considering the predictor $f'$ defined by \eqref{eq:f'}.
It can be shown that when $\beta\in [-1,1]$ is chosen properly, 
updating $f$ to $f'$ decreases the expected loss by more than $\gamma^2 = \alpha + \varepsilon$. 
In \Cref{thm:sequence}, if $f$ belongs to some $\mF_n$, then the updated predictor $f'$ belongs to $\mF_{n+k}$. 
Therefore, if $f\in \mF_n$ is not $\gamma$-multicalibrated, 
$f$ is not $(\alpha + \varepsilon)$-loss-optimal w.r.t.\ the \emph{larger} class $\mF_{n+k}$.
To prove \Cref{thm:sequence}, we need to show that $f$ is not $\varepsilon$-loss-optimal w.r.t.\ the \emph{current} class $\mF_n$.

Let $\opt_n$ denote the minimum expected loss achievable by predictors from $\mF_n$.
If $\opt_n \le \opt_{n+k} + \alpha$, then $f\in \mF_n$ being not $(\alpha+\varepsilon)$-loss-optimal w.r.t.\ $\mF_{n+k}$ implies that it is not $\varepsilon$-loss-optimal w.r.t.\ $\mF_n$, as desired. Thus we only need to worry about cases where $\opt_{n}$ is larger than $\opt_{n+k} + \alpha$. We observe that $\opt_n$ is non-increasing in $n$ and it is bounded in $[0,1]$, so there can only be a few choices of $n$ for which $\opt_n$ is larger than $\opt_{n+k}$ by a significant amount $\alpha$. Excluding these bad choices of $n$ allows us to prove \Cref{thm:sequence}. A detailed proof is presented in \Cref{sec:analysis}.

\subsection{Lower Bound}
\label{sec:intro-lower}
A natural question is whether the bound on the number of unlucky choices of $n$ can be improved. While we are not able to show that the bound in \Cref{thm:informal-main} is optimal, we show tightness of the bound in \Cref{thm:upper-bound} for juntas. 
We prove the following lower bound showing tightness up to constant: there are indeed $\Omega(k/\alpha)$ integers $n$ for which the loss-optimal junta is not $\sqrt \alpha$-multiaccurate. The right quantitative bound uses the noise stability of the majority function \citep[see][]{o2014analysis}. Thus, a stronger result for neural networks than \Cref{thm:informal-main}, if it exists, would use properties of neural networks beyond what our analysis currently uses.

\begin{theorem}[Informal statement of Theorem \ref{thm:lower-bound}] 
For every $k$ and small enough $\alpha$, there exist $m\in \mathbb Z_{> 0}$, a distribution $\mD$ over $\X\times \{0,1\} = \pmo^m\times \{0,1\}$, and at least $\Omega( k/\alpha)$ distinct non-negative integers $n$ such that any model $f_n\in \junta_n$ is \emph{not} $(\junta_k^*,\sqrt{\alpha})$-MA.
\end{theorem}

\subsection{Multicalibration from Structural Risk Minimization}
\label{sec:srm}
There are a few unlucky choices of $n$ for which \Cref{thm:informal-main,thm:sequence} do not guarantee multicalibration from loss minimization. We show that by adding an appropriate regularization term to the loss minimization problem, we can avoid such unlucky choices of $n$. This regularized approach can be viewed as a form of structural risk minimization, and it can be motivated economically as follows. 

Deploying large predictors is expensive. To deploy a finite predictor, one may consider a cost, say, a linear cost of for deploying a NN of $n$ nodes. Once amortized over queries, this is equivalent to using regularization and choosing a predictor that minimizes loss plus a constant, say $1/N$, times size. Since the squared loss is bounded in $[0,1]$, the optimal solution is always an NN with at most $N$ nodes. Our theorem below shows that large NNs optimized for loss, with sufficiently small regularization, achieve MC. 
\eat{These networks may be similarly large.\lunjia{I'm not sure I understand what ``similarly'' means in the last sentence.}\pnote{Me neither. Lets drop the sentence?}}

\begin{theorem}[Informal version of Theorem \ref{thm:reg}]
    Let $k \in \N$ be an integer and $\alpha \in (0,1]$. Then, for some $N_0 = O(k/\alpha)$ and any $N\ge N_0$, selecting a loss-optimal NN with size regularization coefficient $1/N$ will lead to a NN of size $\le N$ that is $\sqrt\alpha$-MC with respect to NNs of size $k$.
\end{theorem}

\subsection{Related Work}
The notion of multicalibration for multigroup fairness was introduced in \cite{hkrr2018}, see also \cite{kearns2017preventing, KleinbergMR17}. This notion has proved to be unexpectedly rich, with connections to several other areas. The work of  \citep{OI} relates it to indistinguishability, while \citep{kim2022universal} connects it to domain adaptation. A line of work on omniprediction \citep{omni, op2,hu2022omnipredictors, performative-omni, GJRR23} has shown that multicalibration for a class $\mC$ implies strong loss minimization guarantees when compared against benchmarks from $\mC$. This is the opposite direction to the one we study here: it shows settings where loss minimization results from multicalibration. A more recent work \citep{op3} in fact shows an equivalence between multicalibration and certain swap loss minimization problems. Without getting into details, the model of loss minimization is inspired by internal regret in online learning, and is different from the standard notion that we study here. 

The relation between loss minimization and multicalibration has been investigated extensively in the literature over the last few years. 
In addition to what we have discussed earlier in the introduction,
it is known that minimizing any proper loss for linear models over a base hypothesis class yields multiaccuracy w.r.t.\ the base class \citep{op2}, but not multicalibration \citep{gopalan2021multicalibrated}. 
\cite{kim2019multiaccuracy} showed that retraining the last layer of a DNN with cross-entropy guarantees multiaccuracy with respect to the features in the penultimate layer; this can be seen as a consequence of the result of \cite{op2} for linear models.

There is a long history studying calibration (not MC) which can be viewed as a special case of MC where the groups are defined by $f(x)$ only, i.e., $c(f(x))$, ignoring the features $x$ itself. 
In particular, a set of practical studies that have found that large NNs are often calibrated ``out of the box'' despite being optimized solely for loss \citep{minderer2021revisiting,hendrycksaugmix,karandikar2021soft,desai2020calibration,carrell2022calibration}. 
The work of \citet{calibgap} attempts to explain this phenomenon by proving that there is a tight connection between the {\em smooth calibration error} \citep{kakade-foster,utc} and the loss reduction that is obtainable by using Lipshcitz post-processing of the predicitons. The upshot is that models whose predictions cannot be improved by such post-processing have small calibration error. They speculate that this explains why some large NNs are calibrated out of the box.
\eat{

functions whose loss cannot be improved much by (a simpler) post-processing are exactly those that are close to calibrated.
They then speculate that well-trained neural networks are likely to be loss-optimal in this sense, and thus calibrated.
The present work can  be seen as pursuing a similar line of reasoning, but in the more general setting of multicalibration. }
\eat{
The recent work of \citep{calibgap} studies the tradeoff between calibration error and loss minimization. They show that this tradeoff only manifests when the class $\mC'$ over which we optimize is not sufficiently expressive. They show that for sufficiently powerful $\mC'$, {\em proper} loss minimization does indeed imply calibration. This leaves open the possibility that such a result holds for multicalibration for a suitable choice of $\mC'$. 
}

The work of \cite{kim2019multiaccuracy} showed that the performance of image classifiers on demographic subgroups can be improved by ensuring multiaccuracy.  We are not aware of experimental work measuring the degree of multicalibration for large neural networks on massive training sets. There are numerous works on the representation ability of NNs \cite[e.g.,][]{hornik1989multilayer} and about their accuracy and generalization \citep[e.g.,][]{zhang2021understanding,nagarajan2019uniform}, too large to survey here.

\subsection{Paper Organization} Section \ref{sec:prelims} defines the setting and mathematical preliminaries. Section \ref{sec:analysis} gives our main results, with the regularization analysis covered in Section \ref{sec:reg}. In \Cref{sec:lower} we show a matching lower bound for families of juntas. Finally, Section \ref{sec:discussion} discusses risks and limitations of the present work.

\section{Preliminaries}\label{sec:prelims}

We assume $\X \subseteq \R^d$ is a Euclidean space in some dimension $d$, and focus on binary outcomes (labels) in $\Y= \bits$. We can also allow for $\Y = [0,1]$, in which case the multicalibration guarantee we achieve is called mean multicalibration \citep{jung2020moment}, or the multiclass setting with $l > 2$ distinct labels. We focus on the binary outcome setting for simplicity. We assume an arbitrary, unknown joint distribution $\mD$ on $\X \times \Y$.
For simplicity we focus on the expected squared loss of predictors $f:\X\to [0,1]$ defined as follows:
$$\eloss(f) := \E\nolimits_{(x,y)\sim\mD}\left[\bigl(y-f(x)\bigr)^2\right].$$
Since only one distribution $\mD$ is used throughout the analysis, we assume it is fixed and omit it from our notation, e.g., writing $\eloss(f)$ rather than $\eloss_\mD(f)$ and omitting it from the expectation when clear from context.

We let $\NN_n$ denote the family of fully-connected Neural Networks with the standard ReLU activations and exactly $n$ nodes, which map inputs from $\X$ to output values in $[0,1]$. The ReLU activation computes $\relu(z):=\max(0, z)$ on input $z$. Thus $\NN_n$ is defined by a Directed Acyclic Graph (DAG) with $n$ nodes, where the activation of a node with $d_\text{in}$ inputs is $a=\relu(w \cdot a_\text{in} + b)$ where $a_\text{in}\in\R^{d_\text{in}}$ is the vector of activations of its inputs, $w\in\R^{d_\text{in}}$ is a vector of equal dimension, and $b \in \R$ is a bias term. 
We also define $\NN_0 =\{0\}$, i.e., the constant $0$ function is defined to be computed by a 0-node NN. The set $\NN_n$ is monotonically non-decreasing with $n$ since we can trivially add identity nodes to compute the same function.

Enforcing an output in $[0,1]$ is added to the definition for convenience and can easily be achieved by wrapping the output $z$ in two ReLUs: 
\begin{equation}\label{eq:clip}
\forall z \in \R, \quad \clip(z):=\relu(z-\relu(z-1)) =\max(0,\min(1,z))\in [0,1].
\end{equation}

Similarly to the definition of $\NN_n$, we use $\NN_k^*$ to denote the family of fully-connected Neural Networks with the standard ReLU activations and exactly $k$ nodes, which map inputs from $\X\times [0,1]$ to output values in $[-1,1]$. To get negative output values, we allow the output node to drop the ReLU transformation.

The results in this paper can be extended by enforcing various architectures, such as a NN with a given number of hidden layers, but we present the results for a general feed-forward architecture for simplicity.

\section{Multicalibration from Loss Minimization}\label{sec:analysis}
In this section, we prove \Cref{thm:informal-main,thm:sequence} showing that for the family of neural networks of a given size as well as other families that satisfy closeness under composition, loss minimization over the family implies multicalibration. 

We start with a key observation used in essentially all previous boosting-style algorithms for multicalibration:
\begin{lemma}[{Loss reduction from MC violation \citep[see e.g.][]{hkrr2018,mc-boosting,calibgap}}]\label{lem:standard}
Let $f:\X\to [0,1]$ be a predictor and $c:\X\times [0,1]\to [-1,1]$ be an auditor function. For a distribution $\mD$ over $\X\times \{0,1\}$, define $\beta\in [-1,1]$ by
\[
\beta:=\E\nolimits_{(x,y)\sim \mD}\left[c\bigl(x, f(x)\bigr) \cdot \bigl(y-f(x)\bigr)\right].
\]
Define a new predictor $h:\X\to [0,1]$ such that $h(x):=\clip\bigl(f(x) + \beta c(x,f(x))\bigr)$ for every $x\in \X$. Then $\eloss(h) \le \eloss(f) - \beta^2$.
\end{lemma}
\begin{proof}
Define $g(x):=f(x)+\beta c(x, f(x))$, so $h(x)=\clip(g(x))$. Then,
    \begin{align*}
        \eloss(g) &= \E\left[(y-f(x)-\beta c(x, f(x)))^2\right]\\
        &= \E\left[(y-f(x))^2-2 \beta c(x, f(x))(y-f(x)) + \beta^2 c^2(x, f(x))\right]\\
        &= \eloss(f) -2 \beta \E\left[c(x, f(x))(y-f(x))\right] + \beta^2 \E\left[c^2(x, f(x))\right]\\
        &= \eloss(f) - 2\beta^2 + \beta^2 \E\left[c^2(x, f(x))\right].
    \end{align*}
    Since $c^2(x, f(x)) \le 1$, this implies that $\eloss(g) \le \eloss(f)-\beta^2$. Finally, note that for any $(x, y)$, $(y-h(x))^2\le (y-g(x))^2$ since $y\in \bits$ and $h(x)$ is the projection of $g(x)$ onto the closest point in $\zo$, thus $\eloss(h) \le \eloss(g).$
\end{proof}

We prove \Cref{thm:sequence} using \Cref{lem:standard}:
\begin{proof}[Proof of \Cref{thm:sequence}]
For each $n$, define $\opt_n$ to be the infimum of $\eloss(f)$ over $f\in \mF_n$. We first show that for all but at most $k/\alpha$ choices of $n$, 
\begin{equation}
\label{eq:small-drop}
\opt_n \le \opt_{n + k} + \alpha.
\end{equation}
By our assumption, we have $\mF_0\subseteq \mF_1 \subseteq \cdots$, so $\opt_n$ is non-increasing in $n$. It is also clear that $\opt_n \in [0,1]$. 
For $j\in \{1,\ldots,k\}$, suppose we can find $m$ different nonnegative integers $n_1 < \cdots < n_m$ such that each $n_i$ violates \eqref{eq:small-drop} (i.e., $\opt_{n_i} > \opt_{n_i + k} + \alpha$) and satisfies $n_i \equiv j \ \mathrm{mod}\ k$. It is clear that $n_i + k \le n_{i+1}$, and thus $\opt_{n_i + k} \ge \opt_{n_{i+1}}$. Therefore,
\[
1 \ge \opt_{n_1} - \opt_{n_m + k} \ge \sum_{i=1}^m (\opt_{n_i} - \opt_{n_i + k}) \ge m \alpha,
\]
which implies $m \le 1/\alpha$. That is, there are at most $1/\alpha$ choices of $n$ that violate \eqref{eq:small-drop} and satisfy $n \equiv j \ \mathrm{mod}\ k$. Summing up over $j$, we know that there are at most $k/\alpha$ choices of $n$ that violate \eqref{eq:small-drop}.

It remains to show that for every $n$ satisfying \eqref{eq:small-drop}, for every $\varepsilon\in (0,1)$, every $\varepsilon$-loss-optimal $f\in \mF_n$ is $O(\sqrt{\alpha + \varepsilon})$-MC w.r.t.\ $\mC$. For $c\in \mC$, define $\beta:=\left[c\bigl(x, f(x)\bigr) \cdot \bigl(y-f(x)\bigr)\right]$ and $h(x):=\clip \bigl(f(x) + \beta c(x,f(x))\bigr)$ as in \Cref{lem:standard}. By our assumption, $h\in \mF_{n+k}$. Therefore, by \Cref{lem:standard},
\[
\beta^2 \le \eloss(f) - \eloss(h) \le \eloss(f) - \opt_{n+k} \le \eloss(f) - \opt_n + \alpha \le \varepsilon + \alpha.
\]
This implies $|\beta| \le \sqrt{\alpha + \varepsilon}$. Since this holds for any $c\in \mC$, the predictor $f$ must be $(\mC, \sqrt{\alpha + \varepsilon})$-MC.
\end{proof}
We prove \Cref{thm:informal-main} using \Cref{thm:sequence} and the following basic property of neural networks:
\begin{lemma}[Neural networks are closed under composition]
\label{lm:composition}
Let $f\in \NN_n$ and $c\in \NN_k^*$. Define predictor $h:\X\to [0,1]$ such that $h(x):=\clip\bigl(f(x) + \beta c(x,f(x))\bigr)$ for every $x\in \X$. Then $h\in \NN_{n + k + 2}$.
\end{lemma}
\begin{proof}
    Observe that $h(x)=\clip\bigl(f(x) + \beta c(x,f(x))\bigr)$ is computed by 
    \begin{equation}\label{eq:h}
    h(x)=\relu\bigl(f(x)+\beta c(x, f(x)) - \relu(f(x) + \beta c(x, f(x))-1)\bigr). \\
\end{equation}
 as can be seen through Eq.\ (\ref{eq:clip}). Then observe that Eq. (\ref{eq:h}) is indeed a representation of a $(n +k +2)$-node NN: $n$ nodes to compute $f(x)$, $k$ nodes to compute $c(x, f(x))$ and then the 2 additional $\relu$ nodes as described (note that $f(x)$ and $c(x, f(x))$ are both re-used without recomputing them).\footnote{We allow the output node of the $k$-node network computing $c(x,f(x))$ to drop the ReLU transformation because $c(x,f(x))\in [-1,1]$ could be a negative value which cannot be the output of a ReLU transformation. 
Thus the $(n + k + 2)$-node network we construct to compute $h$ may contain an internal node $p$ that does not apply the ReLU transformation. This can be easily fixed by noting that node $p$ computes its output $a$ by  $a = w \cdot a_\text{in} + b$ where $a_\text{in}$ is a vector consisting of outputs of previous nodes. We can remove node $p$ from the $(n + k + 2)$-node network without changing the final output because any node taking $a$ as input can alternatively take $a_{\text{in}}$ as input and reconstruct $a$ using the affine transformation $w \cdot a_\text{in} + b$.}
Thus $h \in \NN_{n +k + 2}$ as claimed.
\end{proof}
\begin{proof}[Proof of \Cref{thm:informal-main}]
The theorem follows immediately by combining \Cref{thm:sequence} and \Cref{lm:composition}.
\end{proof}

\subsection{Regularization}\label{sec:reg}

We show that the regularized approach discussed in \Cref{sec:srm} allows us to avoid the unlucky sizes $n$.
Specifically, we prove the following theorem about general predictor families in the setting of \Cref{thm:sequence} and then specialize it to neural networks in \Cref{thm:reg}.
\begin{theorem}
\label{thm:sequence-reg}
In the setting of \Cref{thm:sequence}, fix $\alpha > 0$ and
consider $n\in \Z_{\ge 0}$ and $f\in \mF_n$ that minimize $\eloss(f) + \alpha n / k$ up to error $\varepsilon$, where $\alpha n / k$ is a regularization term that depends on $n$. That is, 
\begin{equation}
\label{eq:opt-reg}
\eloss(f) + \alpha n / k \le \inf_{n'\in \Z_{\ge 0}, f'\in \mF_{n'}}\eloss(f') + \alpha n' / k + \varepsilon.
\end{equation}
Then $f$ is $(\mC, \sqrt{\alpha + \varepsilon})$-MC.
\end{theorem}
\begin{remark}
\label{remark:reg}
We can always choose $n \le k/\alpha$ in \eqref{eq:opt-reg}. If $n > k/\alpha$, we can replace $n$ by $\tilde n = 0$ and replace $f$ by an arbitrary $\tilde f\in \mF_0$, and get an improvement: $\eloss(\tilde f) + \alpha \tilde n/k \le 1 < \alpha n / k \le \eloss(f) + \alpha n / k$.
\end{remark}
\begin{proof}[Proof of \Cref{thm:sequence-reg}]
Fixing $n' = n + k$ in \eqref{eq:opt-reg}, we have
\[
\eloss(f) + \alpha n / k \le \inf_{f'\in \mF_{n + k}}\eloss(f') + \alpha (n + k)/ k + \varepsilon = \opt_{n + k} + \alpha (n + k)/k + \varepsilon.
\]
This implies that $\eloss(f) \le \opt_{n + k} + \alpha + \varepsilon$. The rest of the proof is identical to the proof of \Cref{thm:sequence}. Specifically, for $c\in \mC$, defining $\beta$ and $h$ as in \Cref{lem:standard}, we get $\beta^2 \le \eloss(f) - \eloss(h) \le \eloss(f) - \opt_{n + k} \le \alpha + \varepsilon$, which implies $|\beta| \le \sqrt{\alpha + \varepsilon}$, as desired.
\end{proof}
Combining \Cref{thm:sequence-reg} and \Cref{lm:composition}, we get the following theorem about neural networks:
\begin{theorem}[Size-regularized NNs]\label{thm:reg}
Let $\mD$ be a distribution over $\X\times \{0,1\}$. For every $k\in \Z_{> 0}$ and every $\alpha > 0$,
consider $n\in \Z_{\ge 0}$ and $f\in \NN_n$ that minimize $\eloss(f) + \alpha n / (k + 2)$ up to error $\varepsilon$. That is,
\[
\eloss(f) + \alpha n / (k + 2) \le \inf_{n'\in \Z_{\ge 0}, f'\in \NN_{n'}}\eloss(f') + \alpha n' / (k + 2) + \varepsilon.
\]
Then $f$ is $(\NN_k^*, \sqrt{\alpha + \varepsilon})$-MC.
\end{theorem}

By \Cref{remark:reg}, the theorem above gives a way of finding an $n \le (k+2)/\alpha$-node NN that is $\sqrt{\alpha + \varepsilon}$-MC with respect to $k$-node NNs.

\paragraph{Finite training sets.} If one has a training set consisting of $m$ examples drawn i.i.d.\ from $\mD$, a natural approach is to minimize training loss. If $m$ is sufficiently large compared to the number of NN parameters, which is polynomial in $n$, generalization bounds guarantee every $f \in \NN_n$ will have training loss close to its true loss $\eloss(f)$, and thus any $\hat{f}_n$ which minimizes training error (still ignoring computation time) will have, say, $\eloss(\hat{f}_n) \le  \min_{\NN_n} \eloss(f) + \varepsilon$ for $\varepsilon = O(\alpha)$. This can, in turn, be used with the above theorems to show that, outside a set of $O(k/\alpha)$ NN sizes, all NNs that are optimal on the training set will be $O(\sqrt\alpha)$-MC with respect to $k$-node NNs. While such an analysis would be straightforward, we find it unsatisfactory given the fact that current generalization bounds, though correct, seem to be too loose to capture the performance of many NN learning algorithms in practice (e.g. \citep{zhang2021understanding,nagarajan2019uniform}).
However, our results above may be relevant in settings where generalization is guaranteed --- for example,
when models are trained using Stochastic Gradient Descent (SGD) for only one-epoch (one pass through the training data).
In this case, the learning algorithm can be thought of as optimizing via SGD directly on the population loss,
without considering a finite train set.
If we heuristically believe that, when run for long enough, the optimization reaches close to a population loss minima within the architecture family, then our results imply most resulting minima will also be multicalibrated.
Moreover, the one-epoch setting may not be far from reality, since most modern large language models
are indeed trained for only one epoch \citep{GPT3,biderman2023pythia}, and
there is evidence that generalization in the multi-epoch setting can be understood
via the one-epoch setting \citep{nakkirandeep}.

\section{A Lower Bound}\label{sec:lower}

We would like to understand the tightness of the analysis we have presented. Unfortunately, this is challenging due to the complex structure of NNs. Instead, this section provides some evidence of the tightness of our analysis at least \textit{using our current methods}. To do this, we consider another natural class of functions namely {\em $k$-juntas} and a weaker notion of multigroup fairness called multiaccuracy (MA), to which our analysis applies. In this setting, we show a sharp result, proving that our bounds are tight up to constant factors.

The class of \textit{juntas}, functions that depend on a small subset of inputs, has been well studied in computational learning theory \citep{MOSSEL2004421, FISCHER2004753}. In this section, we show that juntas satisfy a similar property to NNs in that, for most sizes, minimizing loss for juntas also implies MA with respect to smaller juntas. We also prove a \textit{lower bound} for showing that our bounds are tight up to constant factors.

For our results in this section, we choose the domain $\X$ to be the Boolean cube $\{-1,1\}^m$ for a dimension $m\in \mathbb Z_{>0}$. That is, every $x\in \X$ can be written as $x = (x_1,\ldots,x_m)$ where $x_i\in \{-1,1\}$ for every $i = 1,\ldots,m$. For a positive integer $k$ and a function $f:\X\to \R$, we say $f$ is a \emph{$k$-junta} if there exist $i_1,\ldots,i_k\in \{1,\ldots,m\}$ and $g:\{-1,1\}^k\to \R$ such that $f(x) = g(x_{i_1},\ldots,x_{i_k})$ for every $x\in \X$. We say a function $f:\X\to \R$ is a $0$-junta if $f$ is a constant function. We use $\junta_k$ (resp.\ $\junta_k^*$) to denote the family of $k$-juntas that map inputs in $\X = \{-1,1\}^m$ to output values in $[0,1]$ (resp.\ $[-1,1]$).

We consider multiaccuracy, which is a weaker notion than multicalibration, where the auditor functions $c$ do not have access to the predictions $f(x)$:

\begin{definition}[Multiaccuracy \citep{hkrr2018}]
\label{def:ma}
Let $\gamma > 0$ and $\mC$ be a class of auditor functions $c:\X \to [-1,1]$.
The predictor $f:\X\to [0,1]$ is $(\mC, \gamma)$-multiaccurate or $(\mC, \gamma)$-MA if for all $c \in \mC$,
$$\Big|\E\nolimits_{(x,y)\sim\mD}\left[c(x) \cdot \bigl(y-f(x)\bigr)\right]\Big| \le \gamma.
$$
\end{definition}

We prove the following theorem for juntas which is similar to our result about neural networks in \Cref{thm:informal-main} (restating the theorem from the introduction):

\juntaupper*
The proof works exactly like the proof of \Cref{thm:informal-main} using the following observation. For $f\in \junta_n, c\in \junta_k^*$, and $\beta\in \R$, define function $h:\X\to [0,1]$ as
\[
h(x) = \clip(f(x) + \beta c(x)),
\]
then $h\in \junta_{n + k}$. The following theorem gives a lower bound that matches \Cref{thm:upper-bound} up to a constant factor.
\begin{theorem}
\label{thm:lower-bound}
For every $k\in \mathbb Z_{> 0}$ and every $\alpha\in (0,1/(4\pi^2)]$, there exist $m\in \mathbb Z_{> 0}$, a distribution $\mD$ over $\X\times \Y$ with $\X = \{-1,1\}^m$ and $\Y = \{0,1\}$, and at least $k/(6\pi^2\alpha)$ distinct nonnegative integers $n$ such that any model $f_n\in \junta_n$ is \emph{not} $(\junta_k^*,\sqrt{\alpha})$-MA.
\end{theorem}

Key to our proof of \Cref{thm:lower-bound} is the \emph{majority} function. For an odd positive integer $k$, we define the majority function $\maj:\{-1,1\}^k\to \{-1,1\}$ such that for every $(x_1,\ldots,x_k)\in \{-1,1\}^k$, $\maj(x_1,\ldots,x_k) = 1$ if $x_1 + \cdots + x_k > 0$, and $\maj(x_1,\ldots,x_k) = -1$ otherwise. The following lemma about majority functions follows from standard noise sensitivity bounds \citep[see][]{o2014analysis}. We provide a proof in \Cref{sec:maj} for completeness.

\begin{restatable}{lemma}{majcor}
\label{lm:maj-cor}
For any odd positive integers $m$ and $k$ satisfying $k \le m$,
\[
\E[\maj(x_1,\ldots,x_k)\maj(x_1,\ldots,x_m)] > \frac 2\pi \cdot \sqrt{\frac{k}{m}},
\]
where the expectation is over $(x_1,\ldots,x_m)$ drawn uniformly at random from $\{-1,1\}^m$.\footnote{The constant $2/\pi$ cannot be improved because in the limit where $k, m\to \infty,\sqrt{k/m} \to \rho\in (0,1)$, using the central limit theorem one can show that $\E[\maj(x_1,\ldots,x_k)\maj(x_1,\ldots,x_m)] \to (2/\pi)\arcsin\rho$. Also, it is easy to show that $\E[f(x_1,\ldots,x_k)\maj(x_1,\ldots,x_m)]$ is maximized when $f = \maj$ among all $f:\{-1,1\}^k \to [-1,1]$ \cite[see e.g.][Lemma 13]{GKK08}.}
\end{restatable}
Using the majority function, we define a distribution $\mD$ over $\{-1,1\}^m \times \{0,1\}$ for any odd positive integer $m$ as follows. We first draw $x\in \{-1,1\}^m$ uniformly at random and then set
\[
y = \frac 12(1 + \maj(x))\in \{0,1\}.
\]
We define $\mD$ to be the distribution of $(x,y)$.
\Cref{lm:maj-cor} allows us to prove the following lemma about the distribution $\mD$, which we then use to prove \Cref{thm:lower-bound}.
\begin{lemma}
\label{lm:lower-helper}
Let $k,m$ be odd positive integers satisfying $k \le m$. 
Define $\X:= \{-1,1\}^m$ and define distribution $\mD$ as above.
Let $\alpha> 0$ be a parameter satisfying
\begin{equation}
\label{eq:assumption-alpha}
\alpha \le \frac{k}{\pi^2 m}.
\end{equation}
Then any function $f\in \junta_{m - k}$ is $\emph{not}$ $(\mC,\sqrt\alpha)$-MA with respect to $\mC = \junta_k^*$.
\end{lemma}
\begin{proof}
It suffices to show that for any function $f\in \junta_{m - k}$, there exist $i_1,\ldots,i_k\in \{1,\ldots,m\}$ such that
\begin{equation}
\label{eq:lower-helper-1}
\E\nolimits_{(x,y)\sim \mD}[(y - f(x))\maj(x_{i_1},\ldots,x_{i_k})] > \frac 1{\pi}\cdot \sqrt{\frac{k}{m}}.
\end{equation}
By the definition of $f\in \junta_{m - k}(\X)$, there exist $j_1,\ldots,j_{m - k}\in \{1,\ldots,m\}$ and $g:\{-1,1\}^{m-k}$ such that 
\[
f(x) = g(x_{j_1},\ldots,g_{j_{m - k}}) \quad \text{for every }x\in \X.
\]
Now we choose distinct $i_1,\ldots,i_k\in \{1,\ldots,m\}\setminus\{j_1,\ldots,j_{m - k}\}$. We have
\begin{equation}
\label{eq:lower-helper-2}
\E\nolimits_{(x,y)\sim \mD}[f(x)\maj(x_{i_1},\ldots,x_{i_k})] = \E[g(x_{j_1},\ldots,g_{j_{m - k}})\maj(x_{i_1},\ldots,x_{i_k})] = 0,
\end{equation}
where the last equation holds because $(x_{i_1},\ldots,x_{i_k})$ is independent from $(x_{j_1},\ldots,x_{j_{m - k}})$. By \Cref{lm:maj-cor} and our choice of distribution $\mD$,
\begin{equation}
\label{eq:lower-helper-3}
\E\nolimits_{(x,y)\sim \mD}[y\,\maj(x_{i_1},\ldots,x_{i_k})] > \frac{1}\pi\cdot \sqrt{\frac{k}{m}}.
\end{equation}
Combining \eqref{eq:lower-helper-2} and \eqref{eq:lower-helper-3} proves \eqref{eq:lower-helper-1}.
\end{proof}
We are now ready to prove our lower bound \Cref{thm:lower-bound}.
\begin{proof}[Proof of \Cref{thm:lower-bound}]
For any positive integer $k$, define $k_1$ to be the largest odd integer that does not exceed $k$. It is easy to verify that $k_1 \ge k/2 > 0$. For any $\alpha\in (0,1/(4\pi^2)]$, choose $m$ to be the largest odd integer smaller than $k_1/(\pi^2\alpha)$. 
Our assumption that $\alpha \le 1/(4\pi^2)$ ensures that $k_1/(\pi^2\alpha) \ge 4k_1$, and thus $m \ge 3k_1$ and $m \ge k_1/(2\pi^2\alpha)$.
Moreover, our choice of $m$ ensures that
\[
\alpha \le \frac{k_1}{\pi^2m}.
\]
By \Cref{lm:lower-helper}, for $\X = \{-1,1\}^m$ and $\Y = \{0,1\}$ there exists a distribution $\mD$ over $\X\times \Y$ such that for every nonnegative integer $n = 0,\ldots, m - k_1$, any model $f_n\in \junta_n$ is \emph{not} $(\mC,\sqrt{\alpha})$-MA with respect to $\mC = \junta_{k_1}^*$. Since $\junta_{k_1}^*\subseteq \junta_k^*$, any $f_n\in \junta_n$ is also \emph{not} $(\mC,\sqrt{\alpha})$-MA with respect to $\mC = \junta_{k}^*$. The number of such integers $n$ is
\[
m - k_1 + 1 \ge 2m/3 \ge (2/3)(k_1/(2\pi^2\alpha)) = k_1/(3\pi^2\alpha) \ge k/(6\pi^2\alpha).\qedhere
\]
\end{proof}

\section{Discussion}\label{sec:discussion}

We start by discussing some limitations of our work. We do not provide any guidance on selecting $k$, which is an important theoretical and practical issue. Second, there is a blowup in the NN size $n$ required to achieve MC with respect to NNs of size $k$. Also, the model we study is highly unrealistic in that we assume that we have unbounded training data and computational resources at training time to achieve near-minimum expected loss over neural networks of a given size. Although current pre-training efforts do involve massive amounts of data and compute, it is not clear that they are or will ever approach what could be done in this limit.

Nonetheless, we feel that it still may be useful to understand the nature of different predictor representations and the tensions that they face. First of all, one conclusion of this work is that one can achieve MC with NNs--as a representation they are not limited in the same way as some other models. In particular, simpler models (which may enjoy other benefits such as interpretability) may face an inherent tension between accuracy and calibration overall and among different groups. 

If one leads a group of practitioners training NNs and one is concerned about multicalibration, our work suggests that incentivizing them to simply minimize loss may, to some extent, be aligned with MC. This is formal in the sense of Lemma \ref{lem:standard}: if one can identify a group for which the current predictor is miscalibrated, then one can further reduce the loss, escaping a local minimum if the predictor was trained with a local optimization algorithm. But this may naturally happen if they were incentivized to minimize loss. If it is not happening naturally, such a ``check'' could be suggested and it would be entirely compatible with the group's loss incentives.

\paragraph{Future work.}
It would be interesting to empirically measure whether NNs used in practice are multicalibrated. If so, this might be viewed as one additional benefit of their use over simpler models. If not, it would be interesting to understand the reasons-- the gaps between theory and practice.

Another direction would be to prove a similar result for other classes of hypotheses.
The key property of neural nets that we use is that the post-processing required to update the predictor can be incorporated into a predictor of larger size. It is not hard to see that other classes such as decision trees and branching programs (parametrized by size) do have this property. But they are not known to have provable algorithms for squared loss minimization. In contrast linear models do admit efficient loss minimization. But they are not powerful enough to be closed under the kinds of post-processing that multicalibration requires. Moreover, they cannot express the kind of predicates that are interesting for multicalibration, such as $x \in S$ and $f(x) = 0.1$.

An intriguing question is whether there is a hypothesis class that is indeed powerful enough to give interesting multicalibration guarantees and is closed under post-processing, but which admits efficient loss minimization. It is also possible that there exists a no-free-lunch flavor of result which rules out such a possibility. We leave formalizing and proving such results to future work.

\paragraph{Risks.} We conclude on a cautionary note. There is a risk that this and other similar work will be misinterpreted as ``fairness/multicalibration comes for free if you just optimize loss in DNNs'' but that is not a valid conclusion for multiple reasons. First of all, multicalibration is only one notion of fairness. It is particularly applicable in settings where \eat{ in which there is no competition for resources,} and the groups of interest are identifiable from $x$.\eat{ If there is competition for resources, e.g., the highest ranked candidates will be offered loans, then increased accuracy on one group may come at the expense of another. }
Multicalibration cannot protect groups of concern that are not identifiable from $x$, e.g., race is a complex feature that may not be explicitly coded or perfectly computable from existing features, and the protection gets weaker for smaller groups. 

If a practitioner was given an \textit{explicit} set of protected groups whose accuracy is to be prioritized, they ought to consider all  steps in the machine learning pipeline and not just optimization: they might adjust their architectures, collect additional data, or expand the feature set so as to increase the family of protected subgroups and performance on those groups. \eat{Second, if one does want to achieve MC, it is possible that the explicit consideration of MC rather than focusing on loss will lead to different conclusions, especially since the machine learning pipeline involves much more than simply optimizing a classifier.} For instance, explicitly investigating multicalibration may lead to discovering certain groups on which it is hard to achieve good calibration, which requires a different type of learning architecture, better datasets or more informative features.

\section*{Acknowledgments}
We thank Ryan O'Donnell and Rocco Servedio for suggesting the current proof of \Cref{lm:maj-cor}.
We also thank Barry-John Theobald for comments on an early draft.
Part of this work was performed while LH was interning at Apple. LH is also supported by Moses Charikar’s and Omer Reingold’s Simons Investigators awards, Omer Reingold’s NSF Award IIS-1908774, and the Simons Foundation Collaboration on the Theory of Algorithmic Fairness.

JB is supported by a Junior Fellowship from the Simons Society of Fellows.

\bibliographystyle{plainnat}
\bibliography{itcs-refs}

\appendix
\section{Correlation between Majority Functions}
\label{sec:maj}
We restate and prove \Cref{lm:maj-cor}.
\majcor*
\begin{proof}
For any $S\subseteq\{1,\ldots,m\}$, we define
\begin{align*}
\widehat\maj_k(S) := \E[\maj(x_1,\ldots,x_k)\prod_{i\in S}x_i].
\end{align*}
By the Plancherel theorem,
\[
\E[\maj(x_1,\ldots,x_k)\maj(x_1,\ldots,x_m)] = \sum_{S\subseteq\{1,\ldots,m\}}\widehat\maj_k(S)\widehat\maj_m(S).
\]
It is easy to verify that if $S$ is not a subset of $\{1,\ldots,k\}$, then $\widehat\maj_k(S) = 0$. Also, by calculating the Fourier coefficients of the majority function \citep{titsworth1962correlation} \citep[see][Section 5.3]{o2014analysis}, we have $\widehat\maj_k(S)\widehat\maj_m(S) \ge 0$ for any $S\subseteq\{1,\ldots,k\}$. Therefore,
\begin{align*}
\E[\maj(x_1,\ldots,x_k)\maj(x_1,\ldots,x_m)] & \ge \sum_{i=1}^k\widehat\maj_k(\{i\})\widehat\maj_m(\{i\})\\
& = k\left(\genfrac(){0pt}1{k-1}{(k-1)/2}/2^{k-1}\right)\left(\genfrac(){0pt}1{m-1}{(m-1)/2}/2^{m-1}\right)\\
& > k \sqrt{\frac{2}{\pi k}}\sqrt{\frac{2}{\pi m}}\\
& = \frac 2\pi \sqrt{\frac km}.
\end{align*}
Here we use the fact that for any odd positive integer $k$,
\[
\genfrac(){0pt}1{k-1}{(k-1)/2}/2^{k-1} = \frac{1\times 3 \times \cdots \times (k-2)}{2\times 4\times \cdots \times (k - 1)} = \sqrt{\frac 1k \cdot \left(\frac 12 \cdot \frac 32 \cdot \frac 34 \cdot \frac 54 \cdot \cdots \cdot \frac{k-2}{k-1}\cdot \frac k{k-1}\right)} > \sqrt{\frac 1k \cdot \frac 2\pi},
\]
where the last inequality follows from the Wallis product formula for $\pi$.
\end{proof}

\section{Generalization to Any Proper Loss}
\label{sec:proper}
In our main results (Theorems~\ref{thm:informal-main}, \ref{thm:sequence}, \ref{thm:upper-bound}, \ref{thm:sequence-reg}, \ref{thm:reg}),
we show that multicalibration can be achieved by loss minimization over many predictor families. 
While we present these results using the squared loss, in this section we show that they can be generalized to any \emph{proper} loss function satisfying basic regularity conditions, including the cross-entropy loss used widely in neural network training. 
The main technical tool we use to establish this generalization is \Cref{thm:proper-reduction} below from the work of \cite{calibgap}. \Cref{thm:proper-reduction} generalizes the loss reduction lemma (\Cref{lem:standard}) to general proper loss functions.

We first formally define proper loss functions. A loss function $\ell$ takes as input a binary outcome $y\in \{0,1\}$ and a prediction value $v\in [0,1]$, and outputs a real number $\ell(y,v)\in \R$. A proper loss function is defined as follows:
\begin{definition}[Proper loss]
Let $V\subseteq[0,1]$ be a non-empty interval.
We say a loss function $\ell:\{0,1\}\times V \to \R$ is \emph{proper} if for every $v\in V$, it holds that $v\in \argmin_{v'\in V}\E_{y\sim\ber(v)}[\ell(y,v')]$.
\end{definition}
One can easily verify that the squared loss $\sq(y,v) = (y - v)^2$ is a proper loss function over $V = [0,1]$, and the cross-entropy loss $\xent(y,v) = -y\ln v - (1- y)\ln(1 - v)$ is a proper loss function over $V = (0,1)$.

Given input $x\in \X$, a neural network trained to minimize the expected cross-entropy loss usually first computes a \emph{logit} $t\in \R$, from which the final prediction $v\in (0,1)$ is obtained through the sigmoid transformation $\sigma$ by $v = \sigma(t) := 1/(1 + e^{-t})$. Given this relationship between $v$ and $t$, the cross-entropy loss on $v$ is the same as the \emph{logistic loss} $\logistic$ on $t$:
\begin{equation}
\label{eq:xent-logistic}
\xent(y,v) = \logistic(y,t) := \ln(1 + e^t) - yt.
\end{equation}
Therefore, instead of considering the cross-entropy loss of the final prediction $v$, we can equivalently consider the logistic loss of the logit $t$. This allows us to avoid complication caused by the sigmoid transformation and focus just on the ReLU network producing the logit. 

The following lemma generalizes this correspondence between the cross-entropy loss and the logistic loss to any proper loss:

\begin{lemma}[{\citep[see][]{calibgap}}]
\label{lm:dual-prediction}
Let $V\subseteq[0,1]$ be a non-empty interval.
Let $\ell:\{0,1\}\times V\to \R$ be a proper loss function. For every $v\in V$, define $\dual(v):= \ell(0,v) - \ell(1,v)$. Then there exists a convex function $\psi:\R\to \R$ such that
\begin{equation}
\label{eq:dual-prediction}
\ell(y,v) = \psi(\dual(v)) - y\, \dual(v) \quad \text{for every $y\in \{0,1\}$ and $v\in V$.}
\end{equation}
We can additionally ensure that $\frac{\psi(t_1) - \psi(t_2)}{t_1 - t_2}\in [0,1]$ for any distinct $t_1,t_2\in \R$. If $\psi$ is differentiable, then $\nabla\psi(t)\in [0,1]$ for every $t\in \R$, and $\nabla\psi(\dual(v)) = v$ for every $v\in V$, where $\nabla \psi$ denotes the derivative of $\psi$.
\end{lemma}
A proof of \Cref{lm:dual-prediction} can be found in the work of \citet{calibgap}. It is based on a connection between proper loss functions and conjugate pairs of convex functions studied by \citet{admissible,elicitation,assessor,loss-structure}.

In \Cref{lm:dual-prediction}, we say $\dual(v)$ is the \emph{dual prediction} corresponding to a prediction value $v\in V$. When $\ell$ is the cross-entropy loss $\xent$, the dual prediction $\dual(v)$ is exactly the logit $t$ corresponding to the prediction $v$.
For any proper loss $\ell$, \Cref{lm:dual-prediction} identifies a corresponding convex function $\psi:\R\to \R$. We define a \emph{dual loss function} $\ell\sps\psi:\{0,1\}\times \R\to \R$ such that 
\begin{equation}
\label{eq:dual-loss}
\ell\sps \psi(y,t) = \psi(t) - yt \quad \text{for every } y\in \{0,1\}\text{ and }t\in \R. 
\end{equation}
This definition of a dual loss function is essentially the definition of the Fenchel-Young loss in the literature \citep[see e.g.][]{multiclass-divergence,learning-fenchel}.
Equation \eqref{eq:dual-prediction} implies the following for any $v\in V$ and the corresponding $t = \dual(v)$:
\begin{equation}
\label{eq:dual-loss-relation}
\ell(y,v) = \ell\sps \psi(y,t).
\end{equation}
This generalizes \eqref{eq:xent-logistic} where $\xent$ is a special case of the proper loss $\ell$, and $\logistic$ is exactly the dual loss $\ell\sps \psi$ for the function $\psi$ obtained from \Cref{lm:dual-prediction}.
A loss function $\ell\sps \psi$ satisfying the relationship in \eqref{eq:dual-loss-relation} has been referred to as a \emph{composite loss} \citep[see e.g.][]{loss-structure,composite}.

Using \Cref{lm:dual-prediction} and \eqref{eq:dual-loss-relation}, to study a general proper loss $\ell$, it suffices to consider a convex function $\psi$ obtained from \Cref{lm:dual-prediction} and the dual loss $\ell\sps\psi$ defined in \eqref{eq:dual-loss}.
The dual loss $\ell\sps \psi$ depends on the dual prediction $t$, so instead of predictors $f:\X\to [0,1]$, it is more convenient to consider \emph{dual predictors} $g:\X\to \R$ that output dual predictions $g(x) \in \R$.
We consider the following definition of multicalibration for a dual predictor $g$. Note that by \Cref{lm:dual-prediction}, we can recover a prediction $v\in V$ from its dual prediction $t = \dual(v)$ by $v = \nabla \psi(t)$, assuming that $\psi$ is differentiable.

\begin{definition}[Multicalibration for dual predictors]
Let $\psi:\R\to \R$ be a differentiable function satisfying $\nabla \psi(t)\in [0,1]$ for every $t\in \R$. 
Let $\gamma > 0$ and $\mC$ be a class of auditor functions $c:\X\times \R \to [-1,1]$.
For a dual predictor $g:\X\to \R$, define $f:\X\to [0,1]$ by $f(x) = \nabla \psi(g(x))$.
We say $g$ is $(\mC, \gamma)$-multicalibrated or $(\mC, \gamma)$-MC if for all $c \in \mC$,
\begin{equation}
\Big|\E\nolimits_{(x,y)\sim \mD}\left[c\bigl(x, g(x)\bigr) \cdot \bigl(y-f(x)\bigr)\right]\Big| \le \gamma.
\end{equation}
\end{definition}
We can now state the theorem below which generalizes the loss reduction lemma (\Cref{lem:standard}) to general proper loss functions. For $\lambda \ge 0$, we say a differentiable function $\psi:\R\to \R$ is $\lambda$-smooth if $|\nabla\psi(t_1) - \nabla\psi(t_2)| \le \lambda |t_1 - t_2|$ for every $t_1,t_2\in \R$. For the cross-entropy loss, the corresponding function $\psi$ is given by $\psi(t) = \ln(1 + e^t)$ and it is $(1/4)$-smooth.
\begin{theorem}[Proper loss reduction from multicalibration violation \citep{calibgap}]
\label{thm:proper-reduction}
Let $\psi:\R\to \R$ be a differentiable function satisfying $\nabla \psi(t)\in [0,1]$ for every $t\in \R$. For $\lambda \ge 0$, assume that $\psi$ is $\lambda$-smooth. For a dual predictor $g:\X\to \R$, define $f:\X\to [0,1]$ by $f(x) = \nabla \psi(g(x))$. Consider an auditor function $c:\X\times \R\to [-1,1]$. For a distribution $\mD$ over $\X\times \{0,1\}$, define $\beta\in [-1,1]$ by
\[
\beta:=\E\nolimits_{(x,y)\sim \mD}\left[c\bigl(x, g(x)\bigr) \cdot \bigl(y-f(x)\bigr)\right].
\]
Define a new dual predictor $g':\X\to \R$ such that $g'(x) = g(x) + (\beta /\lambda) c(x,g(x))$. Then,
\[
\E\nolimits_{(x,y)\sim \mD}\ell\sps \psi(y,g'(x)) \le \E\nolimits_{(x,y)\sim \mD}\ell\sps \psi(y,g(x)) - \beta^2/(2\lambda).
\]
\end{theorem}
We omit the proof of \Cref{thm:proper-reduction} as it can be found in the work of \cite{calibgap}. We now use \Cref{thm:proper-reduction} to generalize \Cref{thm:sequence} to general proper loss functions.
Our proof of \Cref{thm:sequence} relies on the basic fact that the squared loss is bounded in $[0,1]$, but many proper loss functions including the cross-entropy loss do not have a similar boundedness property. Nevertheless, we can still generalize \Cref{thm:sequence} by making a weaker assumption about the loss function. For the cross-entropy loss $\xent$ and its corresponding dual loss $\ell\sps \psi$ (the logistic loss), the following holds if we choose $t_0 = 0$ and $B = \ln 2$:
\begin{equation}
\label{eq:init}
\ell\sps \psi(y,t_0) \le \ell\sps \psi(y,t) + B \quad \text{for every }y\in \{0,1\} \text{ and } t\in \R,
\end{equation}
because $\ell\sps \psi(y,t_0) = \ln 2$ when $t_0 = 0$, and $\ell\sps \psi(y,t) \ge 0$ for every $y\in \{0,1\}$ and $t\in \R$. For a general dual loss $\ell\sps\psi$, the assumption that \eqref{eq:init} holds for some $t_0\in \R$ and $B\in \R_{\ge 0}$ is a weaker assumption than boundedness.

Assuming \eqref{eq:init}, we can extend our main result \Cref{thm:sequence} to general proper loss functions in \Cref{thm:sequence-proper} below. For a distribution $\mD$ over $\X\times \{0,1\}$ and a class $\mG$ of dual predictors $g:\X\to \R$, we say $g\in \mG$ is $\varepsilon$-loss-optimal w.r.t.\ a dual loss $\ell\sps\psi$ if
\[
\E\nolimits_{(x,y)\sim \mD}[\ell\sps \psi(y,g(x))] \le  \inf_{g'\in \mG}\E\nolimits_{(x,y)\sim \mD}[\ell\sps \psi(y,g'(x))]+ \varepsilon.
\]
\vspace{-1\baselineskip}
\begin{theorem}[Generalization of \Cref{thm:sequence} to any proper loss]
\label{thm:sequence-proper}
Let $\mD$ be a distribution over $\X\times \{0,1\}$. 
Let $\psi:\R\to \R$ be a differentiable function satisfying $\nabla \psi(t)\in [0,1]$ for every $t\in \R$. For $\lambda \ge 0$, assume that $\psi$ is $\lambda$-smooth. Assume that \eqref{eq:init} holds for some $t_0\in \R$ and $B\in \R_{\ge 0}$.
Let $\mC$ be a class of functions $c:\X\times \R\to [-1,1]$.
Let $\mG_0,\mG_1,\ldots$ be families of dual predictors $g:\X\to \R$ satisfying $\mG_0\subseteq \mG_1 \subseteq \cdots$ and that the constant function $g(x) = t_0$ belongs to $\mG_0$.
For some positive integer $k$, assume that for every nonnegative integer $n$, every $g\in \mG_n$, every $c\in \mC$, and every $\beta\in [-1/\lambda,1/\lambda]$, the composed dual predictor $g'$ defined below satisfies $g'\in \mG_{n+k}$:
\begin{equation}
g'(x) = g(x) + \beta c(x,g(x)) \quad \text{for every }x\in \X.
\end{equation}
Then for every $\alpha > 0$, for all but at most $Bk/\alpha$ choices of $n\in \Z_{\ge 0}$, for any $\varepsilon > 0$, every $\varepsilon$-loss-optimal $g\in \mG_n$ w.r.t.\ $\ell\sps\psi$ is $\big(\mC, \sqrt{2\lambda (\alpha + \varepsilon)}\big)$-MC.
\end{theorem}

\Cref{thm:sequence-proper} can be proved similarly to \Cref{thm:sequence}. Specifically, by \eqref{eq:init}, we have $\E[\ell\sps\psi(y,t_0)] \le \E[\ell\sps\psi(y,g(x))] + B$ for any dual predictor $g:\X\to \R$.
If we define $\opt_n$ to be the infimum of $\E[\ell\sps\psi(y,g(x))]$ over $g\in \mG_n$, we have 
\begin{equation}
\label{eq:opt-drop}
\opt_0 \le \E[\ell\sps\psi(y,t_0)] \le \opt_n + B
\end{equation}
by our assumption that the constant predictor $g(x) = t_0$ belongs to $\mG_0$. 
Therefore we have $B \ge \opt_0 - \opt_n$, bounding the total decrease of $\opt_n$ as a non-increasing function of $n$. We can combine this bound with \Cref{thm:proper-reduction} to prove \Cref{thm:sequence-proper}. We omit the details.

Using \Cref{thm:sequence-proper}, we can prove a result for neural networks generalizing \Cref{thm:informal-main} to any proper loss. We use $\widetilde \NN_n$ to denote the family of all functions $f:\X\to \R$ computable by an $n$-node feed-forward network with ReLU activations. We use $\widetilde \NN_k^*$ to denote the family of all functions $c:\X\times \R\to [-1,1]$ computable by a $k$-node feed-forward network with ReLU activations. To make a negative output possible, we allow the output node of a network in $\widetilde \NN_n$ and $\widetilde \NN_k^*$ to drop the ReLU transformation.

\begin{theorem}[Generalization of \Cref{thm:informal-main} to any proper loss]
\label{thm:proper-main}
    Let $\mD$ be a distribution over $\X\times \{0,1\}$. 
    Let $\psi:\R\to \R$ be a differentiable function satisfying $\nabla \psi(t)\in [0,1]$ for every $t\in \R$. For $\lambda \ge 0$, assume that $\psi$ is $\lambda$-smooth. Assume that \eqref{eq:init} holds for some $t_0\in \R$ and $B\in \R_{\ge 0}$.
    Then for every $k\in \Z_{> 0}$ and every $\alpha > 0$, for all but at most $B(k + 1)/\alpha$ choices of $n\in \Z_{> 0}$, for any $\varepsilon\in (0,1)$, every $\varepsilon$-loss-optimal $g\in \widetilde\NN_n$ w.r.t.\ $\ell\sps\psi$ is $\Big(\widetilde\NN_k^*, \sqrt{2\lambda (\alpha + \varepsilon)}\Big)$-MC.
\end{theorem}
The proof of \Cref{thm:proper-main} is analogous to \Cref{thm:informal-main}. We can similarly generalize \Cref{thm:upper-bound,thm:sequence-reg,thm:reg} to any proper loss. We omit the details.
\end{document}